\documentclass[letterpaper, 10 pt, journal, twoside]{IEEEtran}
\usepackage{graphicx}
\usepackage{multicol}
\usepackage{booktabs}
\usepackage{amsmath,amssymb}
\usepackage{amsthm}
\usepackage[utf8]{inputenc}
\usepackage[english]{babel}
\usepackage{multirow}
\usepackage[font=small]{caption}
\usepackage{float}

\usepackage{lettrine}
\usepackage[]{algorithm2e}
\usepackage{algpseudocode}
\usepackage{cite}
\usepackage{filecontents}
\usepackage{lipsum}
\usepackage{color}
\usepackage{esdiff}
\usepackage{epstopdf}
\usepackage[normalem]{ulem}
\usepackage{soul}

\newcommand{\tabincell}[2]{\begin{tabular}{@{}#1@{}}#2\end{tabular}}
\usepackage{subcaption}
\usepackage{xcolor}
\usepackage{colortbl}
\usepackage{balance}
\newtheorem{thm}{Theorem}

\newtheorem{lem}{Lemma}[section]

\newtheorem{defn}{Definition}[section]

\newtheorem{rem}{Remark}

\definecolor{pink}{rgb}{1, 0, 1}
\definecolor{orange}{rgb}{1, 0.7529, 0}
\definecolor{darkgreen}{rgb}{0, 0.8, 0}
\begin{document}

\title{
{\small This paper has been accepted for publication by IEEE Robotics and Automation Letters}\\
SMART: Self-Morphing Adaptive Replanning Tree}

\author{ \begin{tabular}{cccccccccc}
{Zongyuan Shen} & {James P.  Wilson} & {Shalabh Gupta$^\star$} & {Ryan Harvey}  
\end{tabular}\vspace{-21pt}

\thanks{This work was supported by Air Force Research Laboratory. Distribution A. Approved for public release: distribution unlimited. (AFRL-2023-1595) Date
Approved 04-06-2023. \textit{(Corresponding author:
Shalabh Gupta.)}
}

\thanks {The authors are with the School of Engineering, Department of Electrical
and Computer Engineering, University of Connecticut, Storrs, CT 06269
USA (e-mail: zongyuan.shen@uconn.edu; james.wilson@uconn.edu; shalabh.gupta@uconn.edu; ryan.harvey@uconn.edu).}

\thanks{Digital Object Identifier (DOI): 10.1109/LRA.2023.3315210.}

\thanks{The code is available at https://github.com/ZongyuanShen/SMART. Supplementary videos are available at https://www.youtube.com/@linkslabrobotics.}

\thanks{©2023 IEEE. Personal use of this material is permitted. Permission from IEEE must be obtained for all other uses, in any current or future media, including reprinting/republishing this material for advertising or promotional purposes, creating new collective works, for resale or redistribution to servers or lists, or reuse of any copyrighted component of this work in other works.}

}

\maketitle

\begin{abstract}
The paper presents an algorithm, called Self-Morphing Adaptive Replanning Tree (SMART), that facilitates fast replanning in dynamic environments. SMART performs risk-based tree-pruning if the current path is obstructed by nearby moving obstacle(s), resulting in multiple disjoint subtrees. Then, for speedy recovery, it exploits these subtrees and performs informed tree-repair at hot-spots that lie at the intersection of subtrees to find a new path. The performance of SMART is comparatively evaluated with eight existing algorithms through extensive simulations. Two scenarios are considered with: 1) dynamic obstacles and 2) both static and dynamic obstacles. The results show that SMART yields significant improvements in replanning time, success rate and travel time. Finally, the performance of SMART is validated by a real laboratory experiment. 

\end{abstract}
\vspace{-6pt}


\vspace{-6pt}
\section{Introduction}
\IEEEPARstart{T}{ypical} path planning problems in a static environment aim to optimize the path between the start and goal states by minimizing a user-specified cost-function (e.g., travel time)~\cite{Songgupta2019}. However, many real world applications (e.g., airports, factories, malls, offices, hospitals and homes) consist of moving obstacles (e.g., humans, cobots, carts and wheelchairs). It is envisioned that these applications will be increasingly witnessing the role of cobots in supporting humans for various tasks. Fig. \ref{fig:factoryexample} shows a factory scenario, where cobots support the basic operations such as supplying raw materials and tools, disposing off scrap, and floor cleaning~\cite{song2018}. It is desired that these cobots autonomously navigate in dynamic environments while replanning in real-time as needed to achieve: 1) high success rates and 2) low travel times.   

Replanning strategies are characterized as active or reactive. The active strategies predict the future trajectories of moving obstacles~\cite{wang2020eb} to replan the cobot's path; however, their performance degrades in crowded environments where these trajectories are difficult to compute, associate and predict~\cite{hare2020pose}. Therefore, the reactive strategies replan the cobot's path based on the current information. In this regard, this paper presents an algorithm, called Self-Morphing Adaptive Replanning Tree (SMART), that facilitates real-time reactive replanning in dynamic environments for uninterrupted navigation.

\begin{figure}[t]
    \centering
    \includegraphics[width=0.9\columnwidth]{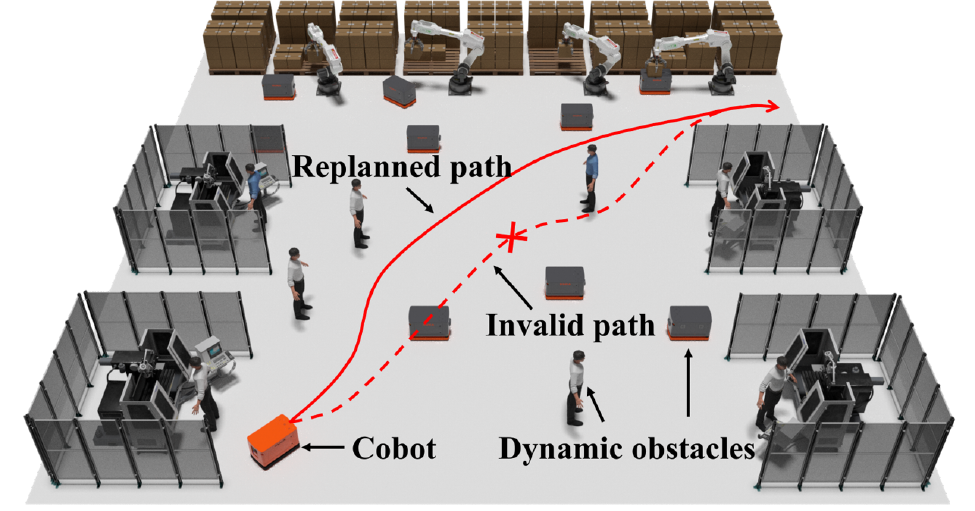}
    \caption{Real-time adaptive replanning in a dynamic factory scenario.}
  \label{fig:factoryexample} \vspace{-18pt}
\end{figure}

\vspace{-6pt}
\subsection{Summary of the SMART Algorithm} 
To initialize, SMART constructs a search-tree using the RRT* algorithm \cite{karaman2011sampling} considering only the static obstacles and finds the initial path. Subsequently, while navigating, the cobot constantly validates its current path for obstructions by nearby dynamic obstacles. If the path is infeasible, SMART performs quick informed replanning that consists of two-steps: 1) tree-pruning and 2) tree-repair. In the tree-pruning step, all risky nodes near the cobot are pruned. This breaks the current tree and forms (possibly) multiple disjoint subtrees. Next, the informed tree-repair step searches for hot-spots that lie at the intersection of different subtrees and provide avenues for real-time tree-repair. 
Then, the utilities of these hot-spots are computed using the shortest-path heuristics. Finally, these hot-spots are incrementally selected according to their utility for merging disjoint subtrees until a new path is found.

\begin{table*}[t]{}
\scriptsize
\caption {Comparison of key features of SMART with other tree-based reactive replanning algorithms. }\label{tab:feature}\vspace{-3pt}
\centering
\setlength\tabcolsep{2.0pt}
\begin{tabular}{l l l l l l l l l} 
 \toprule
 
\specialrule{0em}{1pt}{1pt}\vspace{3pt}
& \textbf{SMART} 
& \textbf{ERRT}('02)\cite{bruce2002real}  
& \textbf{DRRT}('06)\cite{ferguson2006replanning} 
& \textbf{MPRRT}('07)\cite{zucker2007multipartite} 
& \textbf{RRT$^\text{X}$}('16)\cite{otte2016rrtx}
& \textbf{HLRRT*}('19)\cite{chen2019horizon}
& \textbf{EBGRRT}('20)\cite{yuan2020efficient} 
& \textbf{MODRRT*}('21)\cite{qi2020} \\ 
\toprule 

\specialrule{0em}{3pt}{-3pt}
\tabincell{l}{\textbf{Main Tree}\\{\textbf{Root}}} 
& \tabincell{l}{Goal}
& \tabincell{l}{Cobot position} 
& \tabincell{l}{Goal} 
& \tabincell{l}{Cobot position} 
& \tabincell{l}{Goal}
& \tabincell{l}{Cobot position} 
& \tabincell{l}{Cobot position}
& \tabincell{l}{Goal}\\

\specialrule{0em}{3pt}{3pt}
\tabincell{l}{\textbf{Pruning} \\ \textbf{Strategy}} 
& \tabincell{l}{{Prunes risky nodes} \\ {in LRZ. Adds back} \\ {after replanning}}
& \tabincell{l}{{Prunes the} \\ {entire tree}} 
& \tabincell{l}{{Prunes all risky} \\ {nodes and their}\\ {successors}} 
& \tabincell{l}{{Prunes all risky} \\ {nodes}}
& \tabincell{l}{{Assigns infinite} \\{cost to risky} \\{nodes}}
& \tabincell{l}{{Prunes all risky} \\ {nodes and their}\\ {successors}} 
& \tabincell{l}{{Prunes risky path} \\ {nodes and non-} \\ {path successors}}
& \tabincell{l}{{Prunes all risky} \\ {nodes}}\\
\specialrule{0em}{3pt}{3pt}
\tabincell{l}{\textbf{Post-pruning}\\\textbf{Structure}}
& \tabincell{l}{Multiple subtrees} 
& \tabincell{l}{None} 
& \tabincell{l}{Single subtree} 
& \tabincell{l}{Multiple subtrees}
& \tabincell{l}{Graph} 
& \tabincell{l}{Single subtree} 
& \tabincell{l}{Two subtrees}
& \tabincell{l}{Multiple subtrees}   \\

\specialrule{0em}{3pt}{3pt} 
\tabincell{l}{\textbf{Replanning} \\ \textbf{Strategy}} 

& \tabincell{l}{{Reconnects disjoint} \\{subtrees at hot-} \\{spots in an} \\ {informed manner}}

& \tabincell{l}{{Grows a new} \\ {tree by sample} \\ {biasing}} 

& \tabincell{l}{{Regrows remaining} \\ {tree by sample} \\ {biasing}} 

& \tabincell{l}{{Regrows the main} \\ {tree by sample} \\ {biasing}}

& \tabincell{l}{{Graph rewiring} \\ {cascade and repairs}\\ {shortest-path-to-} \\ {goal subtree}}

& \tabincell{l}{{Regrows remaining} \\ {tree by sample} \\ {biasing and lazy} \\ {collision checking}} 

& \tabincell{l}{{Regrows the main} \\ {tree towards goal} \\ {tree by sample} \\ {biasing}} 

& \tabincell{l}{{Reconnects subtree} \\ {roots to main tree} \\ {by feasible} \\ {straight lines}}\\

\specialrule{0em}{3pt}{3pt} 
\tabincell{l}{\textbf{Sampling} \\ \textbf{Strategy}} 

& \tabincell{l}{{Exploits previous}\\ {structure; standard} \\ {sampler if necessary}}

& \tabincell{l}{{Waypoint bias, } \\ {goal bias and } \\ {standard sampler}} 

& \tabincell{l}{{Trimmed-area bias, } \\ {goal bias and } \\ {standard sampler}} 

& \tabincell{l}{{Forest bias, } \\ {goal bias and } \\ {standard sampler}}

& \tabincell{l}{{Standard} \\ {sampler}}

& \tabincell{l}{{GMM-based bias,} \\ {goal bias and } \\ {standard sampler}}

& \tabincell{l}{{Waypoint bias, } \\ {goal bias and } \\ {standard sampler}}  
& \tabincell{l}{None}\\

 \toprule
 \end{tabular}
 \vspace{-18pt}
 \end{table*}

\vspace{-6pt}  
\subsection{Related Work}\label{sec:review}
This section presents a brief literature review of the reactive replanning methods in dynamic environments.

\subsubsection{Tree-based Methods}
Several tree-based replanning methods exist based on different tree pruning and repair strategies. 
Extended RRT (ERRT) \cite{bruce2002real} removes the entire tree when the current path is obstructed and grows a new tree by biasing samples to the previous path. 
Dynamic RRT (DRRT)\cite{ferguson2006replanning} prunes all infeasible nodes and their successors and  regrows the goal-rooted tree biased towards the trimmed area. Multipartite RRT (MPRRT) \cite{zucker2007multipartite} maintains multiple subtrees resulting from node pruning, then reroots the main tree at the cobot's position and reconnects the disjoint tree roots to the main tree by forest biasing. RRT$^\text{X}$ \cite{otte2016rrtx} utilizes a graph to explore the area. When the obstacle information changes, it remodels the search-graph by rewiring cascade and repairs a shortest-path-to-goal subtree to find a new path. Horizon-based Lazy RRT* (HLRRT*) \cite{chen2019horizon} checks the path feasibility within a user-defined time-horizon, and prunes all infeasible nodes and their successors resulting in a single tree, then regrows this tree by biasing samples using a Gaussian mixture model (GMM). Efficient Bias-goal Factor RRT (EBGRRT) \cite{yuan2020efficient} prunes the infeasible path nodes and their non-path successors resulting in a main tree rooted at the cobot's position and a goal tree, then it grows the main tree towards the goal tree. Multi-objective Dynamic RRT* (MODRRT*)\cite{qi2020} connects multiple disjoint tree roots to the goal-rooted tree using feasible straight lines.

There are several differences between SMART and the aforementioned algorithms. First, most algorithms perform node feasibility checking around all detected dynamic obstacles, except HLRRT* which validates the path in a local user-defined horizon. SMART not only restricts path validation but also tree-pruning to the cobot's  neighborhood. Second, some of the above algorithms grow a single (ERRT, DRRT, and HLRRT*) or a  double (EBGRRT) tree-structure after pruning, resulting in repeated exploration of the already-explored area. In contrast, SMART leverages on the maximal tree structure with multiple disjoint subtrees which are incrementally merged during replanning. Third, with the exception of MODRRT*, all above algorithms apply standard sampler for replanning, which could lead to wasteful and slow tree growth. In contrast, SMART performs informed tree-repair at hot-spots.
Table~\ref{tab:feature} shows a comparison of the key features of SMART and other tree-based reactive  replanning algorithms.

\vspace{-0pt}
\subsubsection{Probabilistic Roadmap-based Methods}
The main idea is to construct a road map assuming an obstacle-free space and then update it when the obstacle information is available \cite{leven2002framework}. Both tree-based and probabilistic roadmap based methods have the same time complexity of the processing phase, but the tree-based methods have lower complexity of the query phase, thus making them more suitable for dynamic environments.
 
\subsubsection{Search-based Methods}
D* Lite\cite{koenig2002d} and Lifelong planning A*\cite{koenig2004lifelong} repair an A*-like solution on an underlying graph when the edge costs change given the updated obstacle information. In contrast, SMART finds the hot-spots for fast tree-repair and allows for random sampling if necessary.

\subsubsection{Other Methods}
Some optimization-based methods were proposed such as covariant Hamiltonian optimization for motion planning (CHOMP) \cite{zucker2013chomp}, which use functional gradient techniques to improve the quality of an initial path. Some papers proposed the idea of an escape trajectory as a contingency plan in danger situations\cite{hsu2002randomized,hauser2012responsiveness}. 
A concept of inevitable collision state was proposed in \cite{Fraichard2004inevitable}, where a future collision cannot be avoided. In contrast, SMART identifies critical regions where there is a collision risk and deletes the nodes within. 
Moreover, velocity obstacle based methods\cite{snape2011hybrid} and reinforcement learning based methods \cite{zhu2022collision,fan2020distributed} have also been proposed for incremental planning.

\vspace{-6pt}
\subsection{Contributions}
The paper makes the following contributions: 
\begin{itemize}
\item Development of the SMART algorithm based on fast informed-replanning for real-time dynamic environments.
\item Comprehensive comparison to existing algorithms.
\item Validation using simulation and experimental tests.
\end{itemize}

\vspace{-6pt}
\subsection{Organization}
The remainder of this paper is organized as follows. Section~\ref{sec:smart} presents the details of SMART algorithm and  Section~\ref{analysis} provides the algorithm analysis. Section~\ref{sec:results} shows the simulation and experimental results. Finally, Section~\ref{sec:conclusions} concludes the paper with recommendations for future work.

\vspace{-6pt}
\section{SMART Algorithm}\label{sec:smart}
Let $\mathcal{X}\subset\mathbb{R}^2$ be a region populated by both static and dynamic obstacles. Let $\mathcal{X}_{N} \subset \mathcal{X}$ be the configuration space free of static obstacles. Let $\mathcal{O}=\{O_i: i=1,2,...m\}$ be the set of $m$ dynamic circular obstacles, where $r_i \in \mathbb{R}^+$ is the radius of obstacle $O_i\in \mathcal{O}$, and $x_i(t)  \in \mathcal{X}_{N}$ and $v_i(t) \in \mathbb{R}^+$ denote its position and speed  at time $t \in \mathbb{R}^+$, respectively. Let $\mathcal{R}$ be a circular cobot of radius $r_{\mathcal{R}} \in \mathbb{R}^+$, where $x_{\mathcal{R}}(t) \in \mathcal{X}_N$ and $v_{\mathcal{R}}(t) \in \mathbb{R}^+$ denote its position and speed at time $t$, respectively. Let ($x_{s}$, $x_{g}$) denote the start and goal positions.   

\vspace{-6pt}
\subsection{Initialization}
First, a tiling is constructed on the space $\mathcal{X}$ as defined below. 
\begin{defn}[Tiling]\label{define:Tiling}
A set $\mathcal{C} = \{c_{j} \subset \mathbb{R}^2:j= 1,\ldots |\mathcal{C}|\}$, is a tiling of $\mathcal{X}$, if its elements, called tiles (or cells), have mutually exclusive interiors and cover $\mathcal{X}$, i.e., 
\begin{itemize}
\item  $c^o_{j} \cap c^o_{j'} =\emptyset, \forall {j},{j'} \in \{1,\ldots |\mathcal{C}|\}, {j} \neq {j'}$\\
\item  $\mathcal{X} \subseteq \bigcup_{j=1}^{|\mathcal{C}|}c_{j}$,
\end{itemize} 
where $c^o_{j}$ denotes the interior of cell $c_{j} \in \mathcal{C}$. 
\end{defn}
{Note: The tiling is used only for searching hot-spots (Section~\ref{sec:tree-repair}). The planning and navigation happens in $\mathcal{X}_{N}$.}

SMART is initialized by constructing a RRT*~\cite{karaman2011sampling} tree $\mathcal{T}^0=(\mathcal{N}^0,\mathcal{E}^0)$ rooted at the goal $x_{g}$, where $(\mathcal{N}^0,\mathcal{E}^0)$ denote the sets of nodes and edges. It is recommended that the tree is created such that each cell of the tiling has at least one node to ensure robust and high-quality tree-repair. Each node maintains a data structure including the information about its position, parent, children, tree index, cell index and node status as active or pruned. Initially, all nodes are marked active. Then an initial path is found using  $\mathcal{T}^0$. The path is time embedded to produce the initial trajectory $\sigma^{0}:[t_0,t_f] \rightarrow \mathcal{X}_{N}$, which is a continuous and bounded function, s.t. $\sigma^{0}(t_0) = x_{s}$ and $\sigma^{0}(t_f) = x_{g}$. As the cobot moves,  $\sigma^{0}$ could be blocked by dynamic obstacles $\mathcal{O}$. Thus, the paper presents an informed replanning strategy consisting of two-steps: 1) tree-pruning and 2) tree-repair.

\vspace{-6pt}
\subsection{Tree-Pruning}
\label{prune}
To find a safe trajectory via replanning, it is important to identify and prune the risky nodes of the tree. Since: 1) dynamic obstacles far away from the cobot do not pose an immediate risk, and 2) it is computationally inefficient to do feasibility checking around all detected obstacles, the paper presents a local tree-pruning (LTP) strategy described below. 

\begin{defn}[Local Reaction Zone]\label{define:LRZ}
A local reaction zone \textup{(LRZ)} is defined for the cobot $\mathcal{R}$, such that $\textup{LRZ}_{\mathcal{R}}(t)=\{x\in \mathcal{X}_N:||x-x_{\mathcal{R}}(t)||\leq $ $v_{\mathcal{R}}(t)\times T_{RH}\}$, where $T_{RH} \in \mathbb{R}^+$ is the reaction time-horizon. 
\end{defn}

\begin{defn}[Obstacle Hazard Zone]\label{define:OHZ}
An obstacle hazard zone \textup{(OHZ)} is defined for each dynamic obstacle $O_i \in \mathcal {O}$, such that  $\textup{OHZ}_{i}(t) =\{x\in \mathcal{X}_N: ||x-x_{i}(t)||\leq v_i(t)\times T_{OH} + r_{i} + r_{\mathcal{R}} \}$, where $T_{OH} \in \mathbb{R}^+$ is the obstacle risk time-horizon.
\end{defn}

\begin{defn}[Critical Pruning Region]\label{define:CPR}
Let $\mathcal{D}$ denote all obstacles that are intersecting with the \textup{LRZ} and pose danger to the cobot such that $\mathcal{D}=\{O_i: \textup{LRZ}_{\mathcal{R}}(t)\cap \textup{OHZ}_i(t) \neq \emptyset\}$. Then, the critical pruning region \textup{(CPR)} is defined as
\begin{equation}
\textup{CPR}(t)= \bigcup_{\mathcal{D}} \textup{OHZ}_i(t).
\end{equation}
\end{defn}

Fig.~\ref{fig:example_part1} shows the LTP strategy. The cobot constantly checks the validity of its current path by checking the path nodes and edges in LRZ starting from its current position. According to the LTP strategy, if the current path is invalid, then all tree portions that fall within the CPR are considered to be risky for replanning and are thus pruned. If a node is invalid, it is pruned along with its edges. If an edge is invalid but its two end nodes are safe, then only this edge is pruned. The parent and child information of all affected nodes is updated. Tree-pruning could break the current tree $\mathcal {T}$ into disjoint subtrees.

\begin{defn}[Disjoint Subtree]\label{define:dissub}
A disjoint subtree is a portion of the current tree $\mathcal{T}$ whose root is either the goal or a child of a pruned node or edge. 
\end{defn}

Let $\mathcal{T}_0,\mathcal{T}_1,...\mathcal{T}_{K-1}$ be the $K\in \mathbb{N}^+$ disjoint subtrees formed after pruning, where $\mathcal{T}_0$ is the subtree rooted at the goal. Similarly, let $\mathcal{N}^a$ and $\mathcal{N}^p$ be the sets of alive and pruned nodes, respectively, such that $\mathcal{N}^0=\mathcal{N}^a \bigcup \mathcal{N}^p$.

\begin{figure*}[t]
    \centering
    \subfloat[The current path becomes invalid in the LRZ. Thus, the CPR is identified and the local tree inside the CPR is pruned to form multiple disjoint trees.]{
        \includegraphics[width=0.32\textwidth]{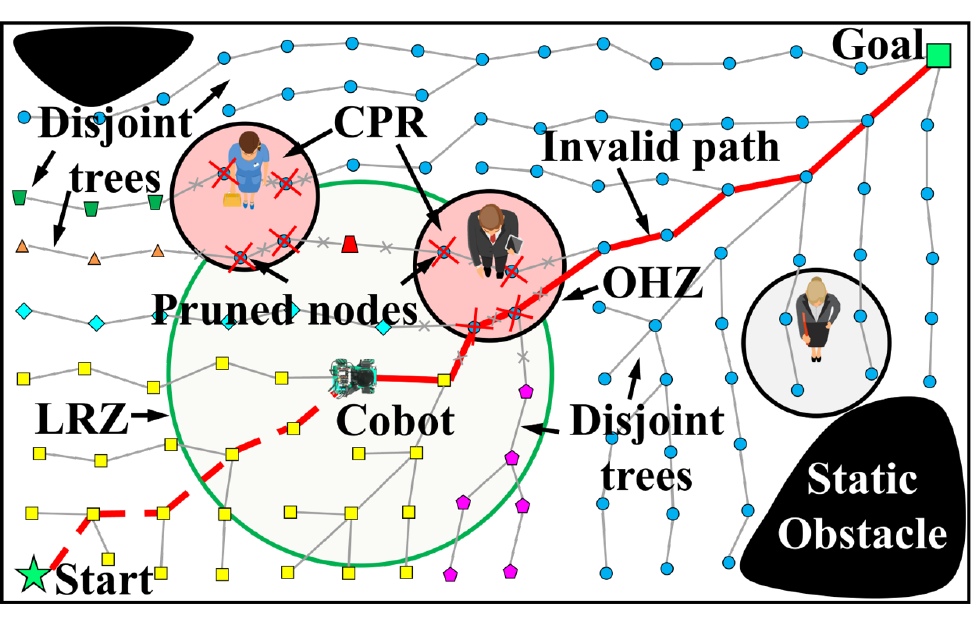}\label{fig:example_part1}}\vspace{0pt}\hspace{-8pt}\quad
    \centering
    \subfloat[The LSR ($\mathcal{S}^3$) is formed around the pruned path node closest to the cobot. Then, the subtree nodes and hot-spots are identified. Hot-spots are color-coded with utilities and the highest one is selected.]{
         \includegraphics[width=0.32\textwidth]{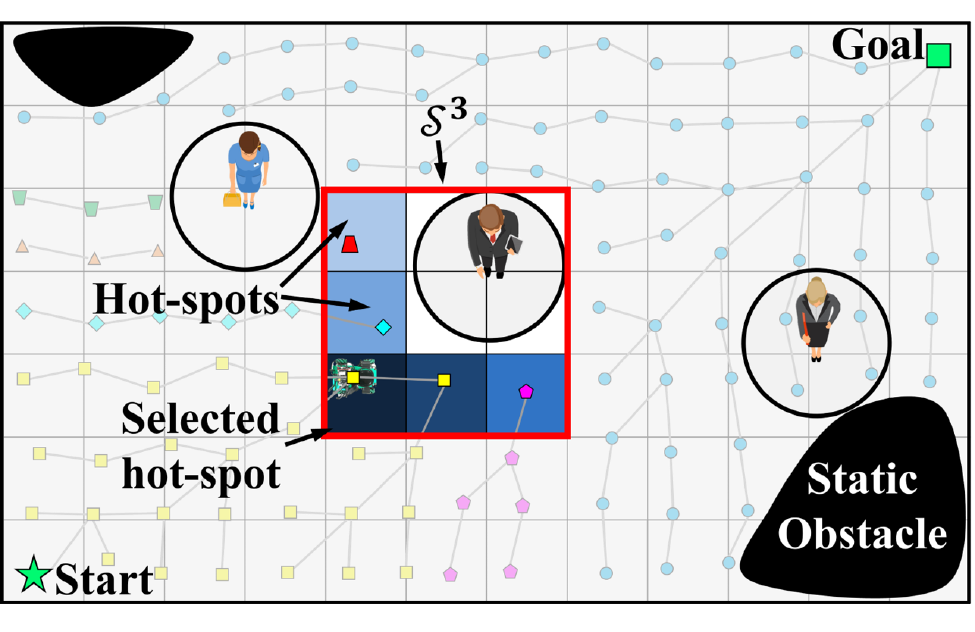}\label{fig:example_part2}}\hspace{-8pt}\quad
    \centering
    \subfloat[The tree is repaired by connecting the nodes of the disjoint subtrees in the neighborhood of the previously selected hot-spot. Then, the map is updated and the new highest utility hot-spot is selected.
    ]{
         \includegraphics[width=0.32\textwidth]{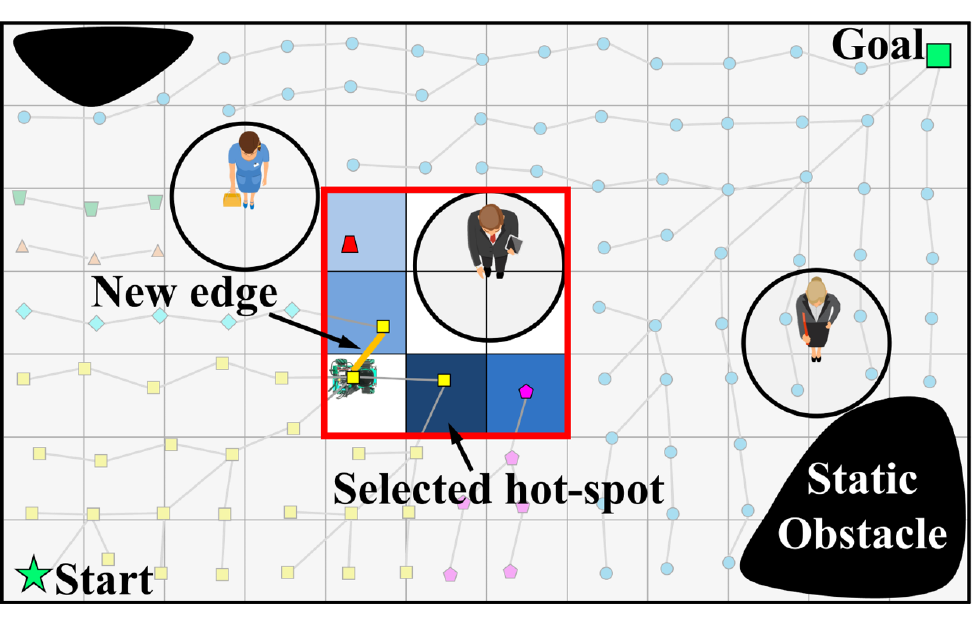}\label{fig:example_part3}}\hspace{-8pt}\\

    \centering
    \subfloat[The tree is further repaired at the selected hot-spot. Then, the map is updated and the new highest utility hot-spot is selected.]{
        \includegraphics[width=0.32\textwidth]{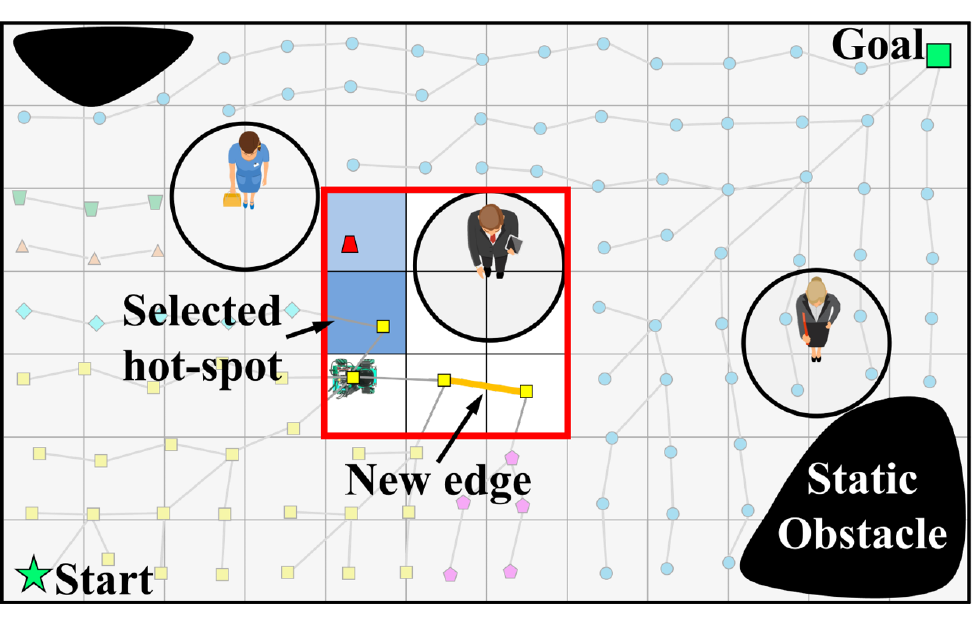}\label{fig:example_part4}}\vspace{0pt}\hspace{-8pt}\quad
    \centering
    \subfloat[The tree is further repaired and the map is updated. Since there are no more hot-spots in $\mathcal{S}^3$ and $\mathcal{T}_0$ is still disconnected, $\mathcal{S}^3$ is expanded to $\mathcal{S}^5$.]{
         \includegraphics[width=0.32\textwidth]{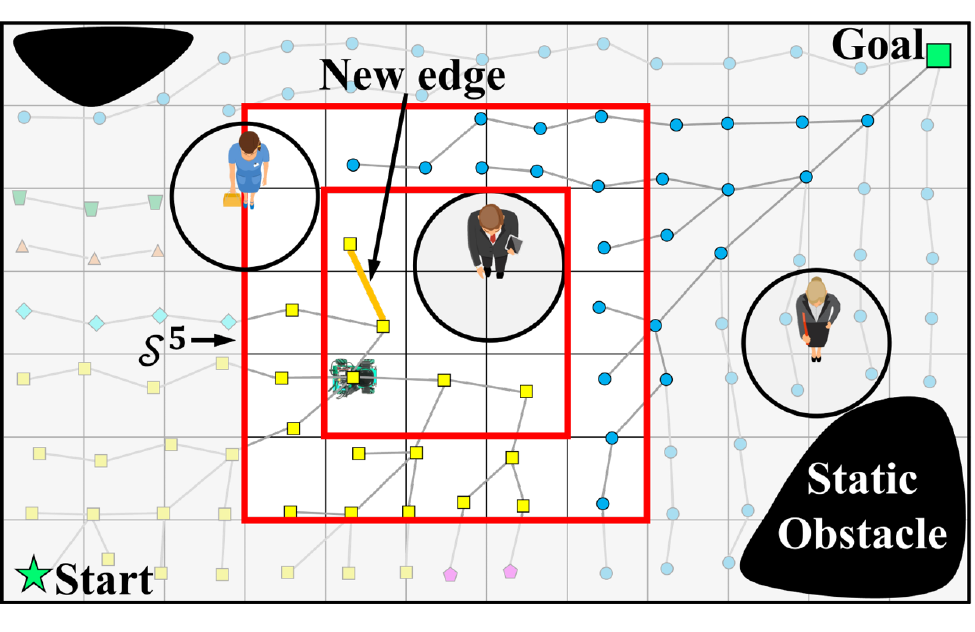}\label{fig:example_part5}}\hspace{-8pt}\quad
     \centering
    \subfloat[Hot-spots are identified in $\mathcal{S}^5$ and color-coded with their utilities. The hot-spot with the highest utility is selected.]{
         \includegraphics[width=0.32\textwidth]{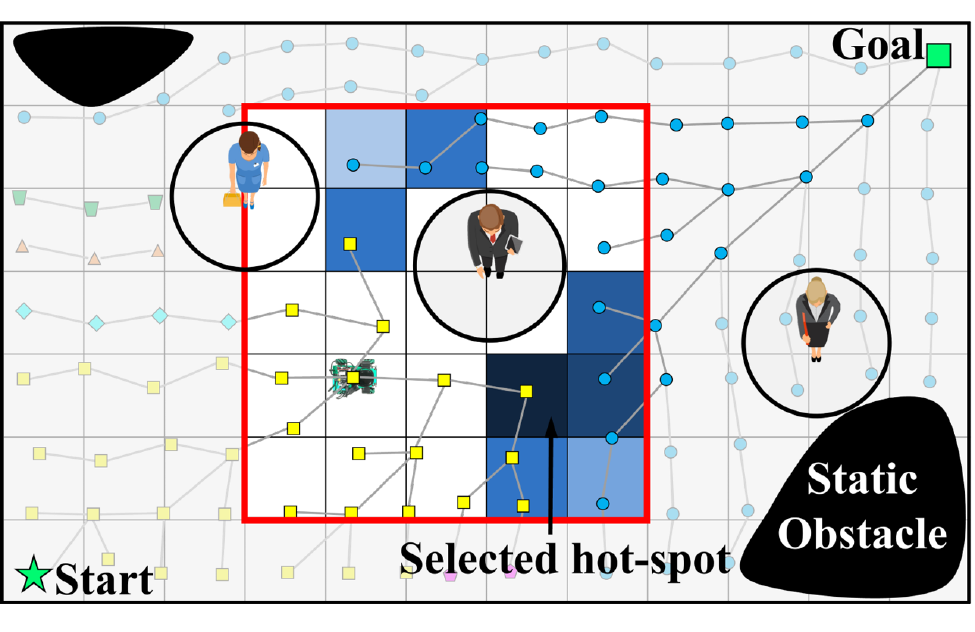}\label{fig:example_part6}}\hspace{-8pt}\\

    \centering
    \subfloat[The tree is further repaired and the map is updated; the goal-rooted subtree $\mathcal{T}_0$ is connected.]{
        \includegraphics[width=0.32\textwidth]{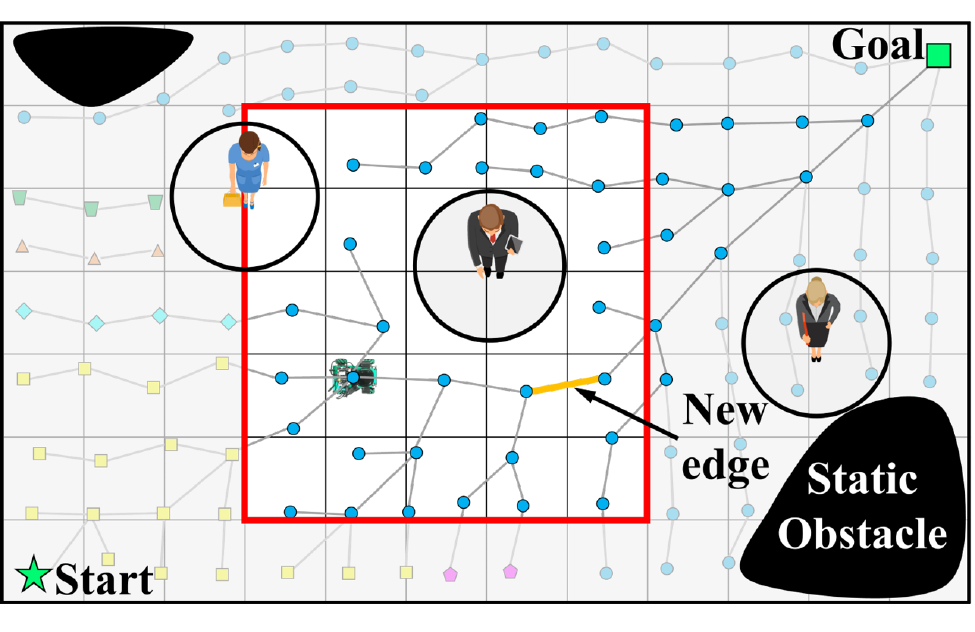}\label{fig:example_part7}}\vspace{0pt}\hspace{-8pt}\quad
    \centering
    \subfloat[The repaired tree is optimized via a rewiring cascade, updating the cost-to-go of its nodes.]{
         \includegraphics[width=0.32\textwidth]{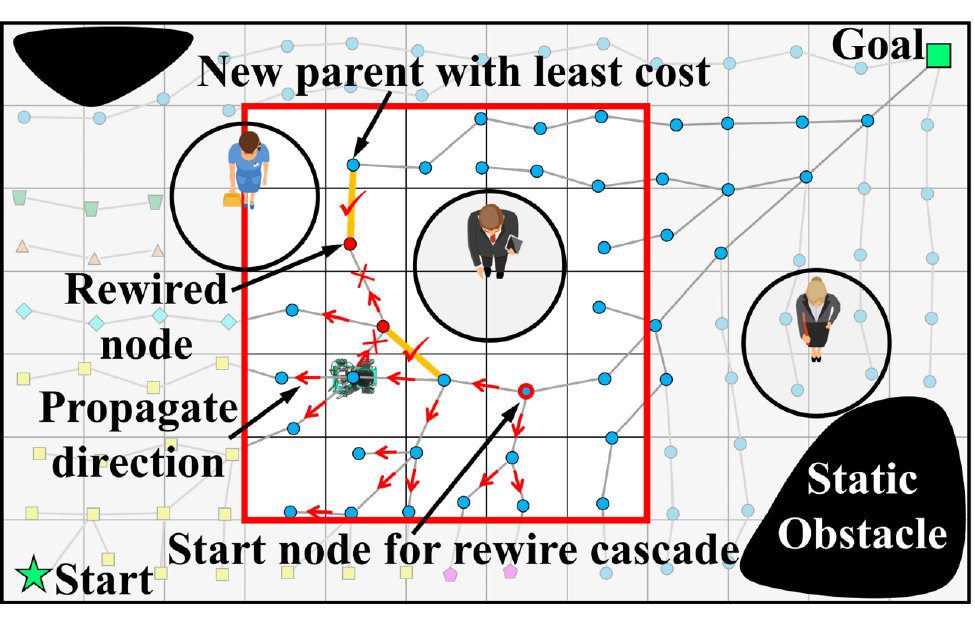}\label{fig:example_part8}}\hspace{-8pt}\quad
    \centering
    \subfloat[Finally, the path is replanned. Then, all pruned nodes are added back and any disjoint trees left are merged to form a single morphed tree $\mathcal{T}$.]{
         \includegraphics[width=0.32\textwidth]{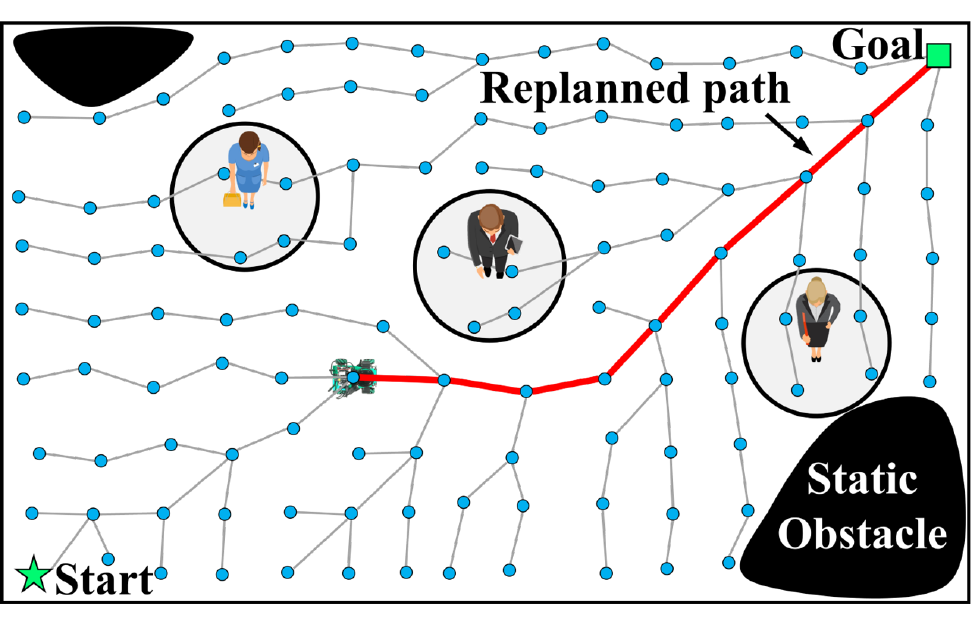}\label{fig:example_part9}}\hspace{-10pt}\\
          \caption{ Illustration of the SMART algorithm: a) tree-pruning and disjoint tree creation, and b)-i) tree-repair and replanning. }\label{fig:example}
          \vspace{-15pt}
 \end{figure*}

\vspace{-6pt}
\subsection{Tree-Repair}\label{sec:tree-repair}
The tree-pruning step is followed by the tree-repair step, which is done to find an updated path to the goal. Let $t_u$ be the time instant at which the current path is blocked and the objective is to find an updated safe trajectory $\sigma:[t_u,t'_f] \rightarrow \mathcal{X}_{N}\setminus \textup{CPR}(t_u)$, s.t. $\sigma(t_u)=x_{\mathcal{R}}(t_u)$ and $\sigma(t'_f)=x_{g}$. 

Fig.~\ref{fig:example_part2}-\ref{fig:example_part9} show an example of the tree-repair process. For fast and efficient tree-repair, SMART exploits the previous exploration efforts by retaining all the disjoint subtrees and pruned nodes for possible future additions. The tree-repair process starts by searching for hot-spots, where the nodes of disjoint subtrees lie in a local neighborhood. Thus, hot-spots are the candidate cells which provide the opportunities for tree-repair by reconnecting the node-pairs of disjoint subtrees.

\begin{defn}[Hot-Spot]\label{define:eligible_cell}
A cell $c \in \mathcal{C}$ is a hot-spot if:
\begin{itemize}
\item [1.] it contains alive nodes of a) at least two disjoint subtrees, or b) a single subtree and at least one of its neighboring cell contains nodes of another disjoint subtree, and
\item [2.] the edge connecting at least one node pair corresponding to the two disjoint subtrees in 1.a) or 1.b) above is not obstructed by the CPR or static obstacles. 
\end{itemize} 
\end{defn}

\vspace{3pt}
\subsubsection{Search for Hot-Spots}\label{search}
The search begins in the vicinity of the damaged path. Thus, the following is defined.
 \begin{defn}[Local Search Region]\label{define:LSR}
Let $\hat{n} \in \mathcal{N}^p$ be the pruned path node closest to the cobot. Let $\hat{n}$ belong to a cell $c_{\hat{n}}\in \mathcal{C}$. Then, the local search region \textup{(LSR)} is defined as the $\ell\times\ell$ neighborhood ($\mathcal{S}^{\ell} \subseteq \mathcal{C}$) of $c_{\hat{n}}$ s.t. $\ell>1$ is an odd number.  
\end{defn}

The search starts in $\mathcal{S}^{\ell}$, $\ell=3$. It is a two-step process.

\vspace{3pt}
a) \textit{Identifying Disjoint Subtrees}:
The first step is to identify and label the disjoint subtrees in $\mathcal{S}^{\ell}$ (Fig.~\ref{fig:example_part1}). To do this, an unlabeled node $n \in \mathcal{N}^a$ is picked within $\mathcal{S}^{\ell}$ and assigned a tree index $k\in \{0,...K-1\}$, where index $0$ corresponds to the goal rooted subtree. Then it is backtracked while labeling all ancestor nodes with the same tree index until reaching either 1) the unlabeled subtree root or 2) a labeled ancestor (or root) from a previous backtracking. In the second case, all nodes visited during backtracking are labeled with the same tree index as that of this ancestor. The above process is repeated until all nodes inside $\mathcal{S}^{\ell}$ are labeled with their tree-indices.

\vspace{3pt}
b) \textit{Identifying Hot-Spots}: The hot-spots are identified by searching all cells and their neighbors within $\mathcal{S}^{\ell}$. Then, a hot-spot map is constructed on $\mathcal{S}^{\ell}$ as follows.  

\begin{defn}[Hot-Spot Map]\label{define:validity}
A hot-spot map is defined on $\mathcal{S}^{\ell}$ such that  
$h^{\ell}: \mathcal{S}^{\ell}\rightarrow \{1, -1\}$, where $1,-1$ denote hot-spot and not a hot-spot, respectively.
\end{defn}


Let $\mathcal{H}^{\ell}\subseteq \mathcal{S}^{\ell} $ denote the set of all hot-spots in $\mathcal{S}^{\ell}$ (Fig.~\ref{fig:example_part2}). If no hot-spot is found within $\mathcal{S}^{\ell}$, s.t. $h^{\ell}(c)= -1, \forall c \in \mathcal{S}^{\ell}$, then the search area is expanded to size $\ell=\ell+2$ and steps a) and b) are repeated until at least one hot-spot is found.

\vspace{3pt}
\subsubsection{Ranking of the Hot-Spots} \label{ranking}
The hot-spots $\mathcal{H}^{\ell}$ are ranked using a utility map (Fig.~\ref{fig:example_part2}) to direct repairing to the region that has the shortest cost-to-come and cost-to-go. 

\begin{defn}[Utility Map]\label{define:utility}
A utility map is defined on  $\mathcal{H}^{\ell}$ such that
$\mathcal{U}^{\ell}: \mathcal{H}^{\ell}\rightarrow \mathbb{R}^+$, where the utility of a cell $c \in \mathcal{H}^{\ell}$ is computed as follows  
\begin{equation}
    \mathcal{U}^{\ell}(c)=\begin{cases}
\frac{1}{\left \|x_{\mathcal{R}}(t_u)-p_c\right \|_2+\displaystyle \min_{n_i\in \mathcal{N}_{c} \cap \mathcal{N}_0}g(n_i)} &  \text{if} \ \mathcal{N}_{c} \cap \mathcal{N}_0 \neq \emptyset \\
\frac{1}{\left \|x_{\mathcal{R}}(t_u)-p_c\right \|_2+\left \|p_c-x_g\right \|_2} &  \text{else},
\end{cases}\vspace{-10pt}
\end{equation}   
\end{defn} where $p_c \in \mathbb{R}^2$ is the centroid of a hot-spot $c \in \mathcal{H}^{\ell}$; $\mathcal{N}_{c}\subset \mathcal{N}^{a}$ is the set of alive nodes inside $c$; $\mathcal{N}_{0}\subset \mathcal{N}^{a}$ is the set of nodes of the goal-rooted subtree $\mathcal{T}_0$; $g(n_i)$ returns the travel cost from node $n_i$ to the goal via the shortest path on $\mathcal{T}_0$. In summary, if the hot-spot contains a node of $\mathcal{T}_0$, then its utility is determined by the heuristic cost-to-come and the actual cost-to-go; otherwise, it is given by the total heuristic cost.

\vspace{3pt}
\subsubsection{Tree-reconnections}\label{repair} This consists of the following steps:
\begin{itemize}
\item [a.] Pick the hot-spot with the highest utility (Figs.~\ref{fig:example_part2}-\ref{fig:example_part4} and \ref{fig:example_part6}). (If there is no hot-spot in $\mathcal{S}^{\ell}$ after tree-repairing, then $\ell=\ell+2$ and go back to~\ref{search} (Fig.~\ref{fig:example_part5})).
\item [b.] Pick an unexamined node from the selected hot-spot. (If all nodes have been examined, then go back to a.)
\item [c.] Pick an unexamined node belonging to a different subtree from the selected hot-spot or its local neighborhood. (If all nodes from different subtrees have been examined, then go back to step b.)
\item [d.] Connect the above node-pair if the edge is feasible (Figs.~\ref{fig:example_part3}-\ref{fig:example_part5} and \ref{fig:example_part7}); otherwise, go back to step c.
\item [e.] Update the  parent-child relationships and subtree indices (Figs.~\ref{fig:example_part3}-\ref{fig:example_part5} and \ref{fig:example_part7}): 1) if any node from the node-pair above belongs to $\mathcal{T}_0$, then this node is set as the parent and all nodes of the connected subtree in $\mathcal{S}^{\ell}$ are assigned the index of $\mathcal{T}_0$ and their cost-to-go are updated, or 2) if none of the nodes in the node-pair belongs to $\mathcal{T}_0$, then one of them is set as the parent and all nodes of the connected subtree in $\mathcal{S}^{\ell}$ are assigned its index. A change in a node's parent is propagated to all its ancestors.
\item [f.] Update the hot-spot and utility maps (Figs.~\ref{fig:example_part3}-\ref{fig:example_part5} and \ref{fig:example_part7}).  
\item [g.]  Check if $\mathcal{T}_0$ is reachable from $x_{\mathcal{R}}$ (Fig.~\ref{fig:example_part7}). If true, then a path can be found to the goal; otherwise, go back to c.

\end{itemize}

\begin{rem}
    If the entire space is exhausted by informed tree-repair and no path is found using alive samples, then random sampling and repairing is done for probabilistic completeness. 
\end{rem}

\RestyleAlgo{ruled}
\LinesNumbered
\begin{algorithm}[t]
\footnotesize
$\{\mathcal{T}^0,\sigma^0\} \leftarrow \textbf{RRT*}(x_{s},x_{g},{\mathcal{X}}_{N})$\tcp*[l]{initialization}
$t=t_0$,  $\mathcal{T}= \mathcal{T}^0$, $\sigma=\sigma_{0}$\; 
\While (\tcp*[h]{goal unreached}) {$x_{\mathcal{R}} \neq x_{g}$}{
    $t \leftarrow \textbf{UpdateClock}()$\;
    $\{x_i(t),v_i(t)\}_{i=1,..m} \leftarrow \textbf{UpdateObstacleState}()$\;
    $\{x_{\mathcal{R}}(t),v_{\mathcal{R}}(t) \} \leftarrow \textbf{UpdateCobotState}()$\;
    \eIf{$\textbf{ValidatePath}(\sigma, \textup{LRZ}_{\mathcal{R}}(t), \{\textup{OHZ}_i(t)\}_{i=1,..m})$}{
    $\textbf{Navigate}(\sigma)$\;
    }{
        $\sigma \leftarrow void$\tcp*[l]{invalid path}
        
        $\{\mathcal{T}_0,...\mathcal{T}_{K-1}\} \leftarrow\textbf{TreePruning} (\mathcal{T},\textup{CPR}(t))$\tcp*[l]{\ref{prune}}
        
        $\ell \leftarrow 1, \mathcal{H}^{\ell} \leftarrow \emptyset, \mathcal{N}_{s} \leftarrow \emptyset$\tcp*[l]{initialize repair}
        
            \While {$\mathcal{T}_0$ is not reachable from $x_{\mathcal{R}}$}{
                \eIf(\tcp*[h]{informed tree-repair}){$\ell < {\ell}_{max}$}{
                    $\ell \leftarrow \ell+2$\;
                    $\mathcal{H}^{\ell} \leftarrow \textbf{HotSpotSearch}(\mathcal{S}^{\ell})$\tcp*[l]{\ref{search}}
                    $\mathcal{U}^{\ell} \leftarrow \textbf{ComputeUtility}(\mathcal{H}^{\ell})$\tcp*[l]{\ref{ranking}} 
                   $\{{\mathcal{N}}_s,{\mathcal{T}}_0 \} \leftarrow \textbf{TreeReconnection}(\mathcal{U}^{\ell}, \mathcal{S}^{\ell})$\tcp*[l]{\ref{repair}}
                }(\tcp*[h]{standard tree-repair}){
                    $x_{rand} \leftarrow \textbf{SampleFree()}$\;
                    Join $x_{rand}$ to all nearby reachable subtrees\;
                    \lIf{$x_{rand} \in \mathcal{T}_0$}{$\textbf{Rewire}(x_{rand})$}
                                        
                }
            }
            $\mathcal{T}_0 \leftarrow\textbf{TreeOptimization}(\mathcal{T}_0,\mathcal{N}_s)$\tcp*[l]{\ref{path}}
            $\sigma \leftarrow \textbf{PathSearch}(x_{\mathcal{R}}, x_g, \mathcal{T}_0)$\tcp*[l]{\ref{path}}
            Add $\mathcal{N}^p$ and disjoint subtrees to $\mathcal{T}_0$ to get a single tree $\mathcal{T}$\;
            
    }
}
\caption{SMART}
\label{alg:SMART} 
\end{algorithm}
\setlength{\textfloatsep}{0pt}

\vspace{-6pt}
\subsection{Tree-Optimization and Path Search}
\label{path} 

During tree-repair, several subtrees were connected to ultimately merge with $\mathcal{T}_0$. These subtrees are optimized using 
rewiring cascade~\cite{otte2016rrtx} starting from all subtree nodes $\mathcal{N}_s$ connected to $\mathcal{T}_0$ (Fig.~\ref{fig:example_part8}). 
Then, an updated trajectory $\sigma$ is found using $\mathcal{T}_0$ (Fig.~\ref{fig:example_part9}). After that, all nodes in $\mathcal{N}^p$ and the roots of the remaining disjoint trees are reconnected with $\mathcal{T}_0$ and a single tree $\mathcal{T}$ is formed. Algorithm~\ref{alg:SMART} describes SMART.

\vspace{-6pt}
\section{Algorithm Analysis}\label{analysis}
Let $\mathcal{W}$ denote the minimal cover of the CPR using the tiling $\mathcal{C}$. Let $\mathcal{N}_{\mathcal{W}}$ be the set of all nodes in $\mathcal{W}$ s.t. $\mathcal{N}^p\subseteq \mathcal{N}_{\mathcal{W}}\subseteq \mathcal{N}^0$.  Let $\overline{n} \geq 1$ be the average number of nodes per cell.
Let the neighbors of a node be defined as all nodes in the same and adjacent cells of that node. Let $h$ be the average number of node neighbors. Let $S^{\widetilde{\ell}}$, $\widetilde{\ell}\leq \ell_{max}$, be the smallest LSR where sufficient hot-spots are found to connect the cobot to $\mathcal{T}^0$.

\begin{lem}
    The tree-pruning complexity is $O(|\mathcal{N}_{\mathcal{W}}|)$. 
    \label{lem:treepruning}
\end{lem}

\begin{proof}  
For tree-pruning, the first step is to identify the CPR for $|\mathcal{D}|$ obstacles that intersect with the LRZ; this has a complexity of $O(|\mathcal{D}|)$.  Next, $\mathcal{W}$ is determined by querying cells using the obstacle locations in the CPR; this has a complexity of $O(|\mathcal{W}|)$. Then, the nodes falling inside the $\mathcal{W}$ cells are queried from a data structure and checked for feasibility; this has a complexity of $O(2|\mathcal{N}_{\mathcal{W}}|)$. 
Since SMART uses a search tree, there are at most $|\mathcal{N}_{\mathcal{W}}|$ edges that need to be checked; this has a complexity of $O(|\mathcal{N}_{\mathcal{W}}|)$. Thus, the overall complexity is $O(3|\mathcal{N}_{\mathcal{W}}|+|\mathcal{W}|+|\mathcal{D}|)$. Since $|\mathcal{N}_{\mathcal{W}}|\geq |\mathcal{W}| \geq |\mathcal{D}|$,  
the overall complexity of tree-pruning  is $O(|\mathcal{N}_{\mathcal{W}}|)$.
\end{proof}

\begin{figure}[t]
    \centering
    \subfloat[Scenario 1: An open space with 10 and 15 dynamic obstacles.]{
        \includegraphics[width=0.450\textwidth]{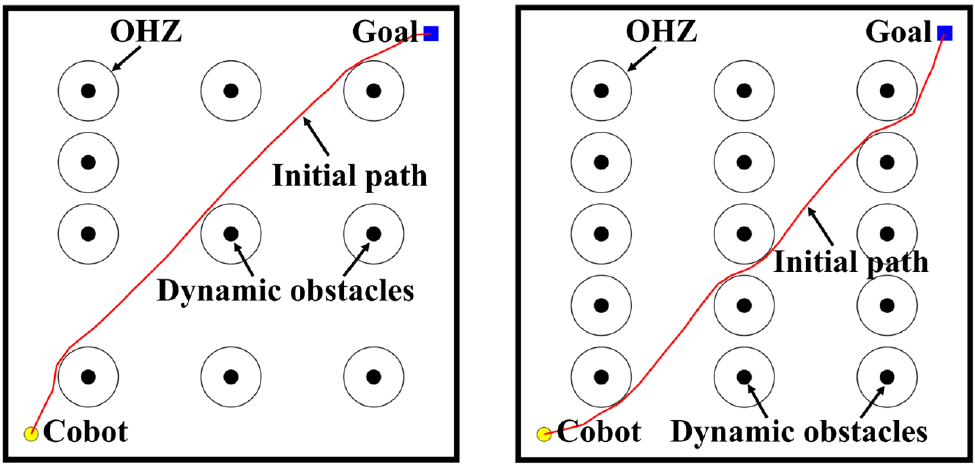}\label{fig:openspace}}\quad  \vspace{6pt}
    \centering
    \subfloat[Scenario 2: A factory with both static and dynamic obstacles.]{
        \includegraphics[width=0.45\textwidth]{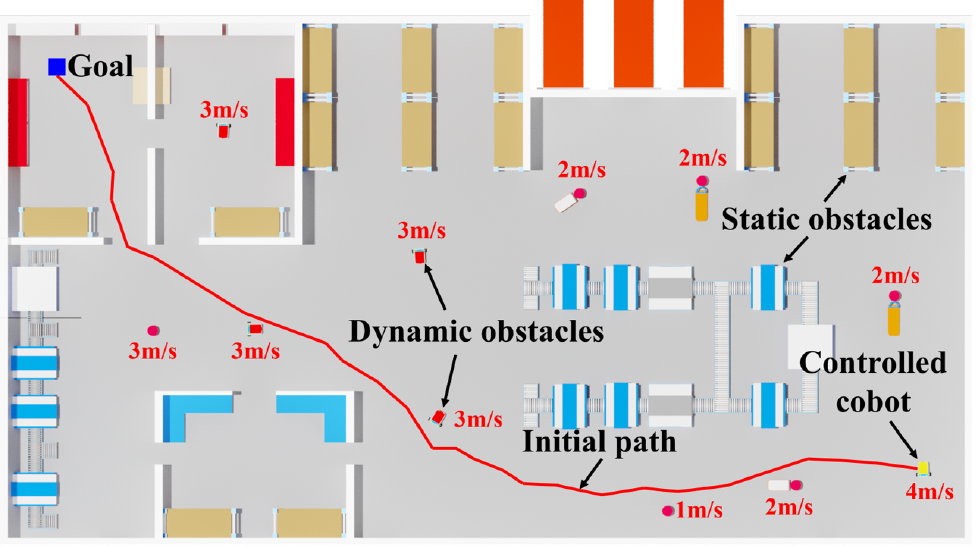}\label{fig:factory}}\\ \vspace{-6pt}    
    \caption{Simulation testing scenarios. }\label{fig:scenario}\vspace{-0pt}   
 \end{figure}

\begin{figure*}[t]
    \centering
    \subfloat[Success rate and average replanning time for Scenario 1 with 10 moving obstacles.]{
        \includegraphics[width=0.95\textwidth]{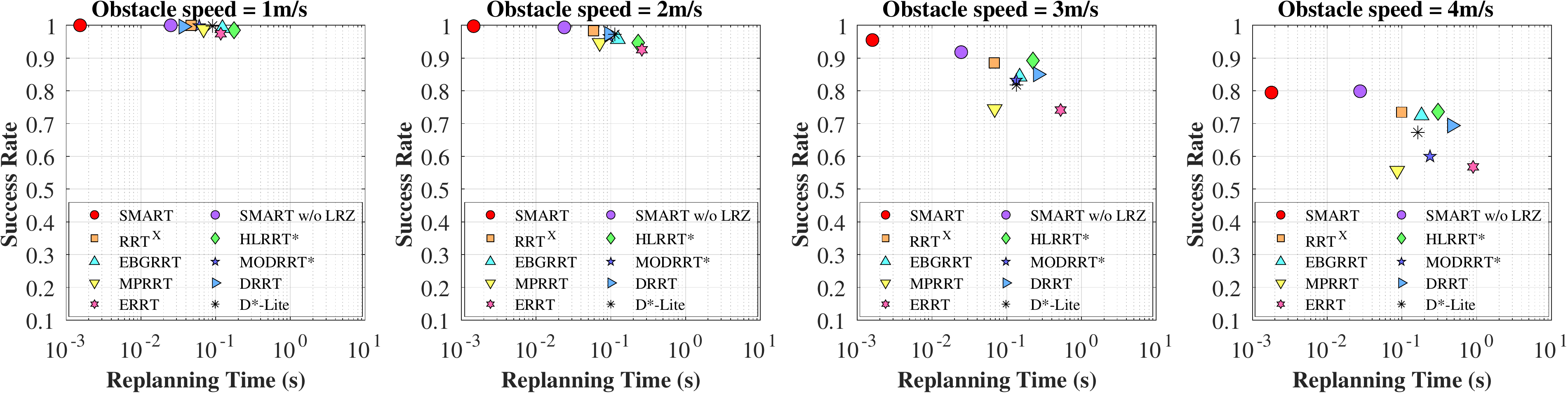}\label{fig:ReplanTime_successRate_part1}}\vspace{5pt}\\
    \centering
    \subfloat[Success rate and average replanning time for Scenario 1 with 15 moving obstacles.]{
         \includegraphics[width=0.95\textwidth]{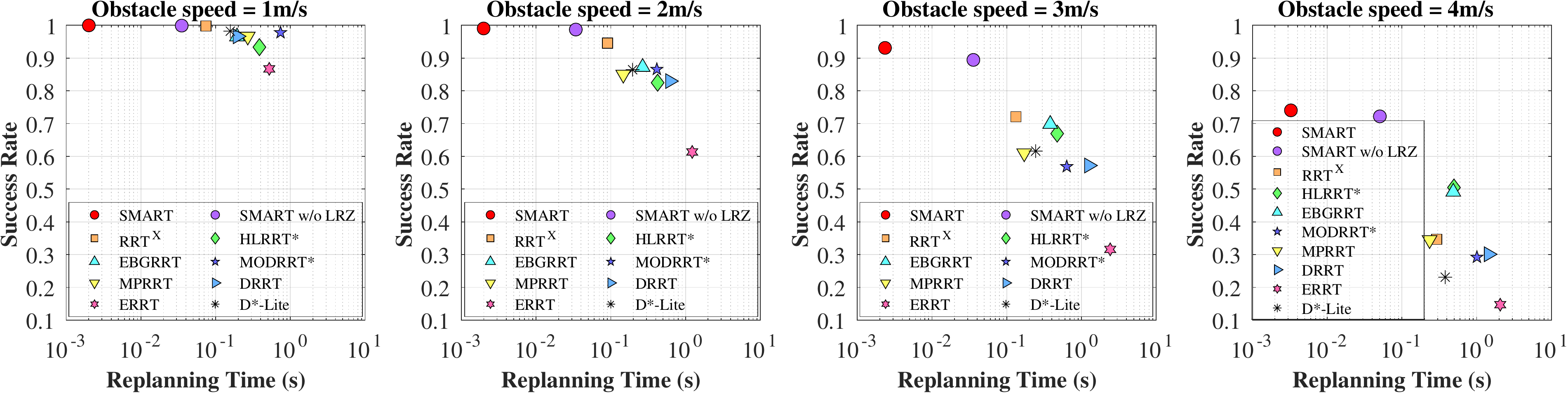}\label{fig:ReplanTime_successRate_part2}}\\
         
          \caption{Comparative evaluation results of success rate and average replanning time for Scenario 1 with a) 10 and b) 15 moving obstacles.}\label{fig:ReplanTime_successRate}
          \vspace{-15pt}
 \end{figure*}

\vspace{-10pt}
\begin{lem}
        The tree-repair complexity is $O(|\mathcal{N}^0|+h\overline{n}|\mathcal{S}^{\widetilde{\ell}}|)$.
    \label{lem:treerepairing}
\end{lem}
\begin{proof} 
\vspace{-9pt}
The first step in tree-repair is to label each node in $\mathcal{S}^{\ell}$ by backtracking, which has a worst-case complexity of $O(|\mathcal{N}^0\setminus\mathcal{N}^p|)\leq O(|\mathcal{N}^0|)$. The next step is to find the sufficient set of hot-spots for tree-repair. The hot-spot status of any cell $c\in\mathcal{S}^\ell$ is checked by comparing the subtree indices of each node in $c$ with up to $h$ neighboring nodes in $\mathcal{S}^\ell$ (Defn.~\ref{define:eligible_cell}-1). For $\overline{n}$ nodes per cell, this has a complexity of $O(h\overline{n})$ per cell. Then, the feasibility of the corresponding edges is checked (Defn.~\ref{define:eligible_cell}-2). This has a complexity of $O(h\overline{n})$ per cell. Thus, the complexity to check the hot-spot status of a cell is $O(2h\overline{n})=O(h\overline{n})$. During the expansion of $\mathcal{S}^\ell$, it becomes clear with little investigation  that the hot-spot status of a cell is checked a maximum of two times. Thus, the complexity to find the sufficient hot-spots is $O(2h\overline{n}|\mathcal{S}^{\widetilde{\ell}}|)=O(h\overline{n}|\mathcal{S}^{\widetilde{\ell}}|)$. Next, the complexity of computing utilities for  $|S^{\widetilde{\ell}}|$ cells (a cell can be a hot-spot up to two times) is $O(2|S^{\widetilde{\ell}}|)=O(|S^{\widetilde{\ell}}|)$. Finally, tree-repair is done via node reconnections. Similar to the hot-spot search, this step finds all neighboring nodes of different subtrees with feasible connecting edges. This has a complexity of $O(h\overline{n})$ per hot-spot.  For $|\mathcal{S}^{\widetilde{\ell}}|$ cells that could each be a potential hot-spot up to two times, this step has a complexity of $O(2h\overline{n}|\mathcal{S}^{\widetilde{\ell}}|)=O(h\overline{n}|\mathcal{S}^{\widetilde{\ell}}|)$. Thus, the overall complexity of tree-repair is $O(|\mathcal{N}^0|+(2h\overline{n}+1)|\mathcal{S}^{\widetilde{\ell}} |)=O(|\mathcal{N}^0|+h\overline{n}|\mathcal{S}^{\widetilde{\ell}}|)$. 
\end{proof}

\vspace{-9pt}
\begin{lem}
The tree-optimization complexity is $O(h|\mathcal{N}^0|)$.
    \label{lem:treeoptimize}
\end{lem} 
\begin{proof}\vspace{-12pt}
In the worst-case scenario, the subtree $\mathcal{T}_0$ initially contained only the goal node, and all disjoint subtrees $\mathcal{T}_1,...\mathcal{T}_{K-1}$ were reconnected to $\mathcal{T}^0$ during tree-repairing. Thus, the rewiring cascade must propagate through $|\mathcal{N}^0|-1$ nodes. Since there are $h$ neighbors per node, the overall complexity for tree-optimization is $O(h(|\mathcal{N}^0|-1))=O(h|\mathcal{N}^0|)$.
\end{proof}

\begin{thm}
    The complexity of SMART is $O(h|\mathcal{N}^0|)$.
    \label{thm:complexity}
\end{thm}
\begin{proof}\vspace{-6pt}
     From Lemmas~\ref{lem:treepruning}-\ref{lem:treeoptimize}, the complexity is $O(|\mathcal{N}_\mathcal{W}|+h\overline{n}|\mathcal{S}^{\widetilde{\ell}}|+(h+1)|\mathcal{N}^0|)$.  
     Since $|\mathcal{N}_\mathcal{W}|\leq |\mathcal{N}^0|$ 
     and $\overline{n}|\mathcal{S}^{\widetilde{\ell}}| \leq |\mathcal{N}^0|$,  the complexity reduces to $O(h|\mathcal{N}^0|)$. 
\end{proof}

\vspace{-6pt}
\begin{thm} SMART informed replanner is complete with respect to the sample-based representation of the environment. 
\end{thm}
\begin{proof} 
If a path exists using alive nodes $\mathcal{N}^a$, then by Defn.~\ref{define:eligible_cell} the hot-spot nodes are necessary and sufficient for replanning. While the incremental hot-spot search guarantees that all hot-spots will be found, the tree-repair sequentially merges and relabels all adjacent subtrees at hot-spots. This guarantees that the cobot will eventually connect to $\mathcal{T}_0$ and a path is found.  \end{proof} 

\begin{figure}[t]
    \centering
    \subfloat[Travel time for Scenario 1 with 10 moving obstacles.]{
         \includegraphics[width=0.43\textwidth]{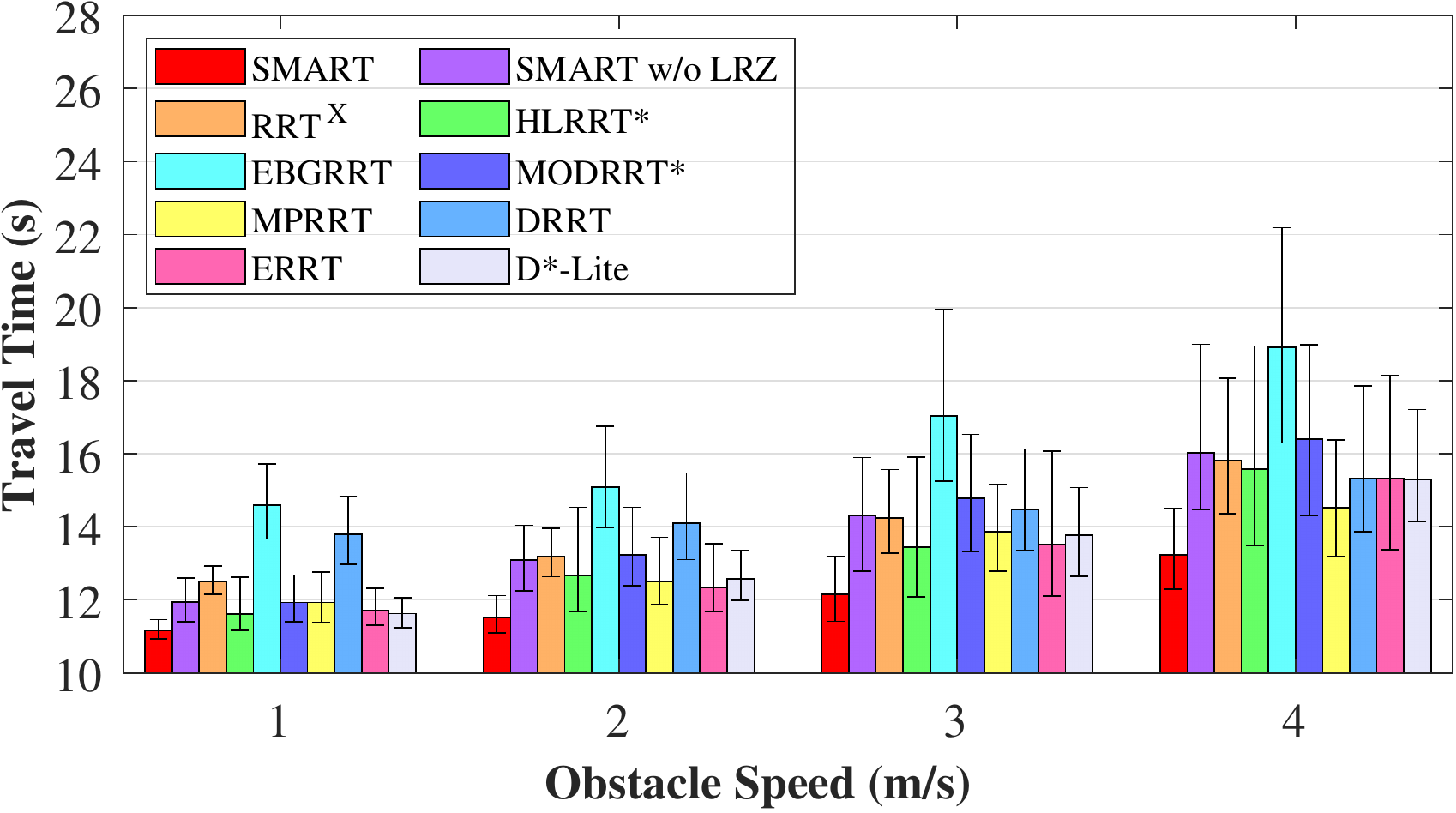}}\quad \\
     \centering
    \subfloat[Travel time for Scenario 1 with 15 moving obstacles.]{
         \includegraphics[width=0.43\textwidth]{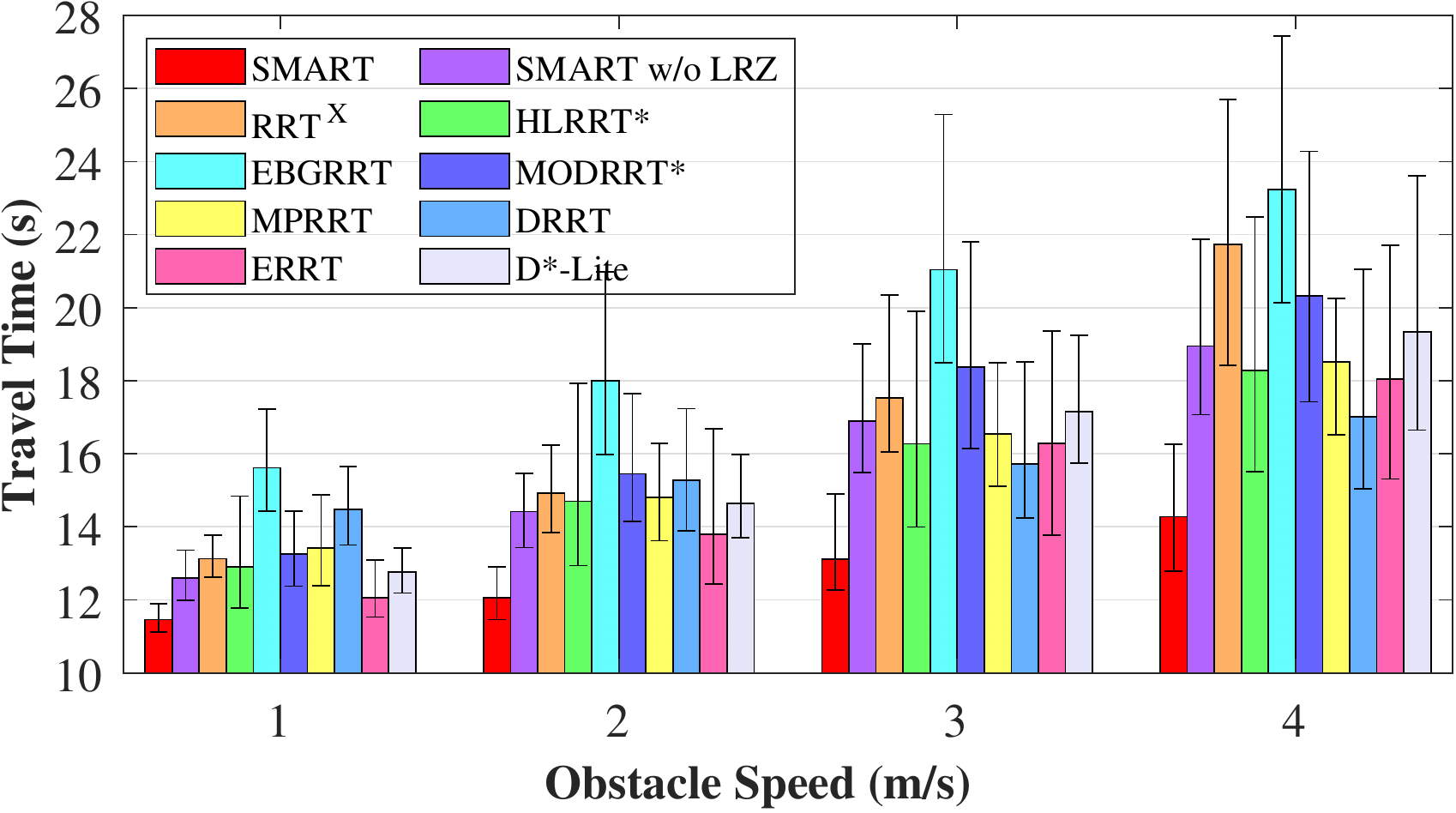}}\\
                  \caption{ Comparative evaluation results of travel time of successful trials for Scenario 1 with a) 10 and b) 15 moving obstacles. The plots show the median and the $25^{th}$ and $75^{th}$ percentile values.} 
          \label{fig:travelTime}
          \vspace{0pt}
 \end{figure}

\vspace{-15pt}
\section{Results and Discussion}
\label{sec:results}
This section presents the testing and validation results of the SMART algorithm via: 1) simulation studies on scenarios containing a) dynamic obstacles and b) both static and dynamic obstacles, and 2) real-experiments in a laboratory.

\vspace{-6pt}
\subsection{Validation by Simulation Experiments}

The performance of SMART is comparatively evaluated with existing methods (Table~\ref{tab:feature}) by extensive simulations. 

\subsubsection{Simulation Set-Up}
SMART is implemented on a holonomic cobot of radius $0.5m$ that moves at a speed of $v_{\mathcal{R}}=4m/s$. For simplicity, the cobot is treated as a point and its radius is added to the static and dynamic obstacles for collision checking. The dynamic obstacles of radius $0.5m$ move along a random heading from $[0,2\pi]$ for a random distance from $[0,10m]$, after which a new heading and distance are selected. Two different scenarios are considered:

$\bullet$ \textit{Scenario 1 with Dynamic Obstacles:} This scenario consists of a $32m \times 32m$ space populated with only dynamic obstacles (Fig. \ref{fig:openspace}). Two cases are conducted including a) $10$ and b) $15$ obstacles. Each obstacle moves at the same speed selected from the set $\{1,2,3,4\}m/s$, resulting in $8$ different combinations of the number of obstacles and speeds. For each combination, $30$ different obstacle trajectories are generated to intersect the cobot, resulting in a total of $240$ case studies.

$\bullet$ \textit{Scenario 2 with Static and Dynamic Obstacles:} 
This scenario depicts a real situation (e.g., a factory) with both static and dynamic obstacles (Fig.~\ref{fig:factory}). It consisted of a $66m \times 38m$ space with a static obstacle layout and $10$ dynamic obstacles. Each obstacle moves at a different speed selected from the set $\{1,2,3,4\}m/s$. Then, $30$ different obstacle trajectories are generated to intersect the cobot, resulting in $30$ case studies.

For each scenario above, $100$ trials are performed for each of the aforementioned cases for each algorithm. All algorithms were deployed in C++ and run on a computer with 2.60 GHz processor and 32 GB RAM. For the same trial, a fixed random seed is used for sample generation for all algorithms, and their initial search trees are of the same size. A trial is marked as failed and the travel time is not recorded if the cobot collides with any obstacle. For SMART, the tiling is generated with cell size of $1m \times 1m$. Based on simulation studies, the reaction time-horizon is set as $T_{RH} = 0.8s$ and an obstacle risk time-horizon of $T_{OH}=0.4s$ is added to all dynamic obstacles for all algorithms. If the cobot moves into an OHZ, then the OHZ is ignored but the actual obstacle is considered for collision checking. The goal bias and random sample rates shared by most algorithms (except RRT$^\text{X}$ and MODRRT*) are set based on~\cite{chen2019horizon} as 0.1 and 0.2 respectively. All other algorithm-specific parameters are set based on the corresponding papers.

To consider sensor uncertainties, noise was injected into the range, the heading angle, and the position of cobot during simulation. A typical lidar (e.g., RPLIDAR S2L\cite{robot}) provides an accuracy of $0.03m$, and a modestly priced compass can provide an accuracy of 1$^\circ$\cite{paull2013auv}. Based on these, the uncertainties were simulated as uniform distributions $U_{[-0.03m,0.03m]}$ and $U_{[-1^\circ,1^\circ]}$ for range and heading, respectively. Similarly, an indoor localization system (e.g., Hagisonic StarGazer) provides a precision of $0.02m$\cite{lopez2013evaluating}. Thus, the localization uncertainty is simulated as a uniform distribution $U_{[-0.02m,0.02m]}$.

\subsubsection{Performance Metrics} 
The following metrics are considered for comparative performance evaluation:
\begin{itemize}
\item Replanning time: Time to replan a new path. 

\item Success rate: Fraction of successful runs out of the total.

\item Travel time: Time from start to goal without collision. 
\end{itemize}

\subsubsection{Simulation Results}\label{subsec:mc_sim_res}

Fig.~\ref{fig:ReplanTime_successRate} shows the comparative evaluation results on Scenario 1. Overall, SMART achieves significant improvements over other algorithms in success rate and replanning time in all case studies. This follows from the facts that i) tree-pruning not only reduces collision checking to nearby obstacles but also produces less number of disjoint trees for repairing, and ii) tree-repair exploits the disjoint subtrees and facilitates repairing at hot-spots for speedy recovery. Fig.~\ref{fig:travelTime} shows that SMART achieves the lowest travel times because of i) lowest replanning time and ii) infrequent replanning. Furthermore, to investigate the value of the tree-repair step, we present an ablation study, where LRZ is removed, thus pruning all risky nodes. Fig.~\ref{fig:ReplanTime_successRate} and Fig.~\ref{fig:travelTime} show that SMART w/o LRZ still performs significantly better than all other algorithms in replanning time and success rate.

$\bullet$ \textit{Effect of Obstacle Speed:} As seen in Fig.~\ref{fig:ReplanTime_successRate}, while the replanning time of SMART is minimally affected by obstacle speed, the success rate dips when obstacle speed is greater than $2m/s$. This is because high-speed moving obstacles have higher chance to hit the cobot. Moreover, as shown in Fig.~\ref{fig:travelTime}, the travel time goes up with the obstacle speed because high-speed obstacles cause frequent replannings.

\begin{figure}[t]
        \centering
         \includegraphics[width=0.450\textwidth]{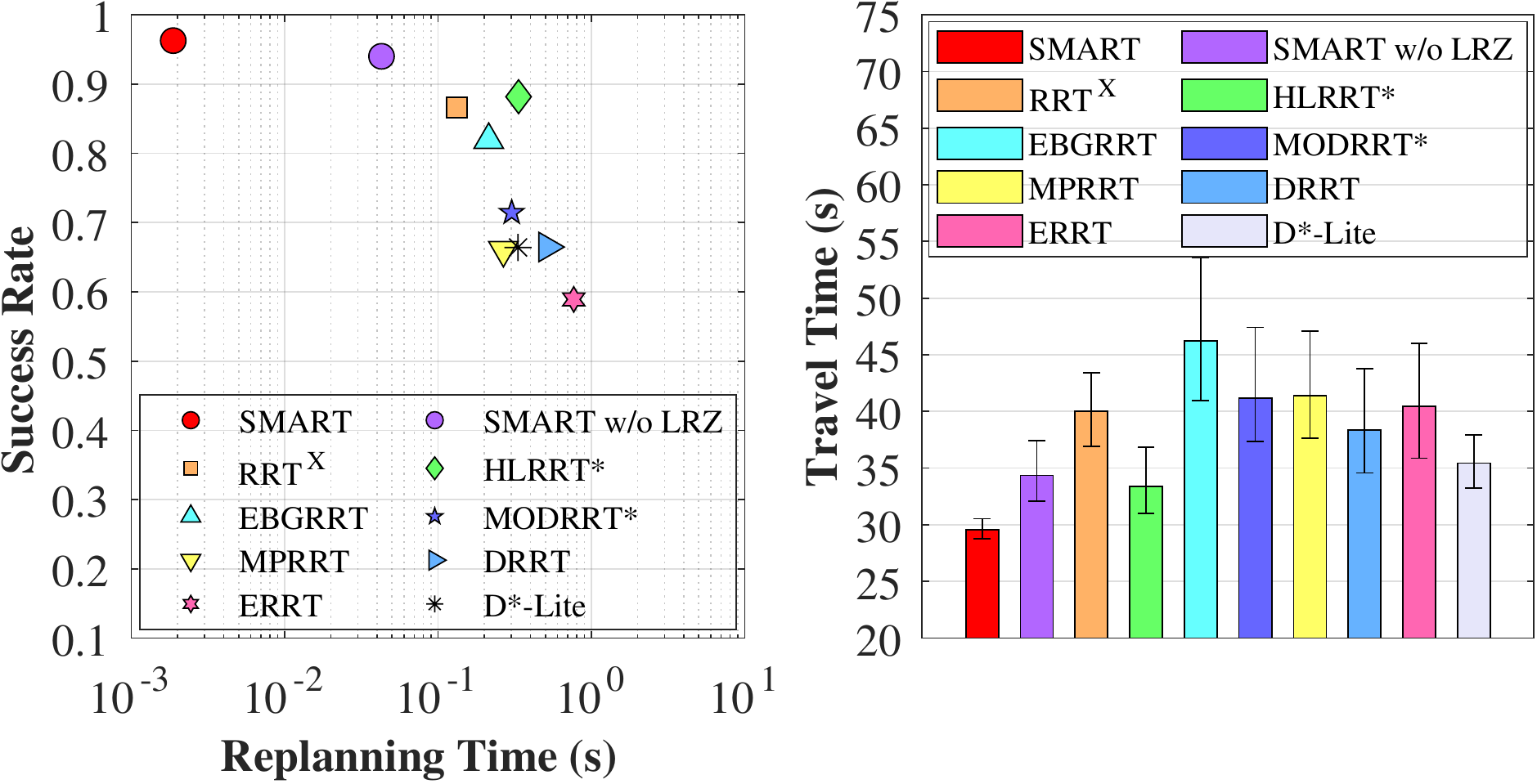}
    \caption{Comparative evaluation results for Scenario 2. }\label{fig:factory_result}\vspace{-0pt}      
 \end{figure}

$\bullet$ \textit{Effect of Number of Obstacles:} As seen in Fig.~\ref{fig:ReplanTime_successRate}, the replanning time increases slightly with an increase in the number of obstacles because of more i) disjoint subtrees for repairing and ii) complex environment with smaller free space. The success rate dips with obstacle number for high obstacle speed because longer replanning time and crowded environment increase collision probability. Similarly, as shown in Fig.~\ref{fig:travelTime}, the travel time goes up with obstacle number because of frequent replannings and complex environment.

\begin{figure*}[t]

    \centering
    \subfloat[Initial path is blocked. Successfully replanned in real time.]{
        \includegraphics[width=0.44\textwidth]{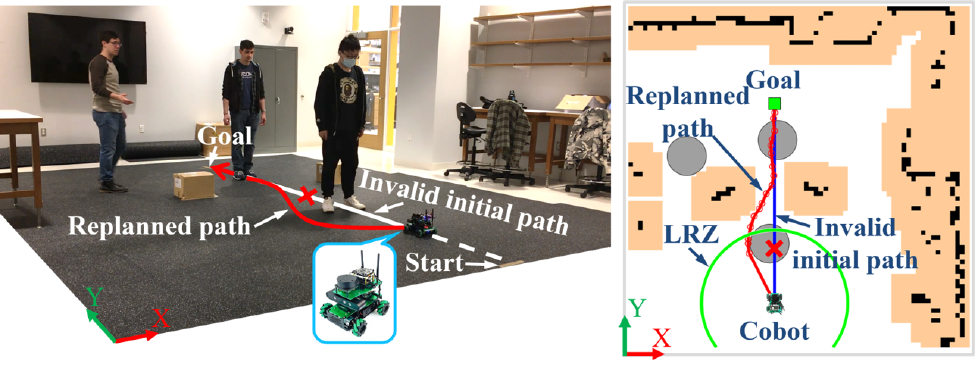}\label{fig:experiment_result_part1}}\hspace{-3pt}\quad
    \centering
    \subfloat[Current path is blocked again. Successfully replanned in real time.]{
         \includegraphics[width=0.44\textwidth]{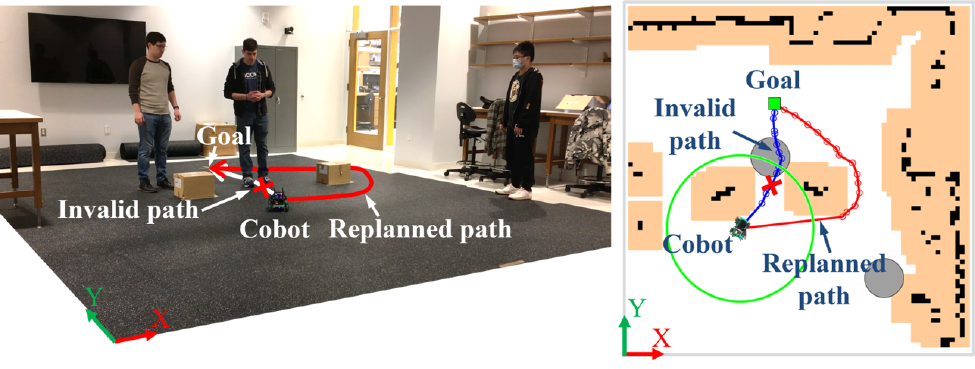}\label{fig:experiment_result_part2}}\hspace{-3pt}\\

    \centering
    \subfloat[Current path is blocked again. Successfully replanned in real time.]{
        \includegraphics[width=0.44\textwidth]{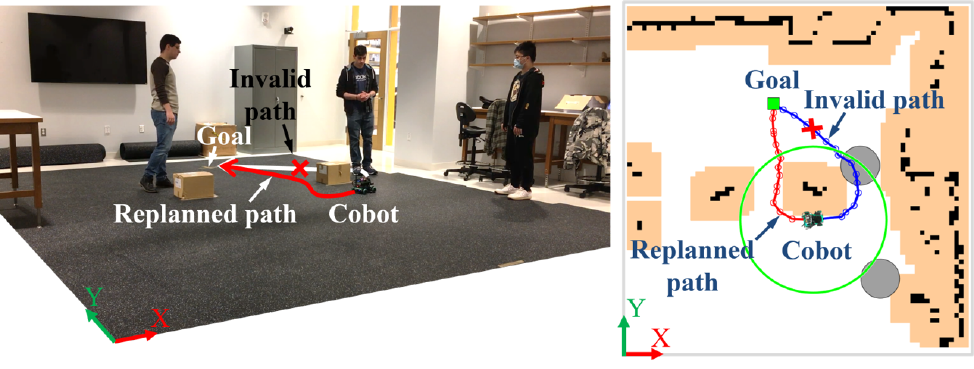}\label{fig:experiment_result_part3}}\hspace{-3pt}\quad
    \centering
    \subfloat[Current path is blocked again. Successfully replanned in real time.]{
         \includegraphics[width=0.44\textwidth]{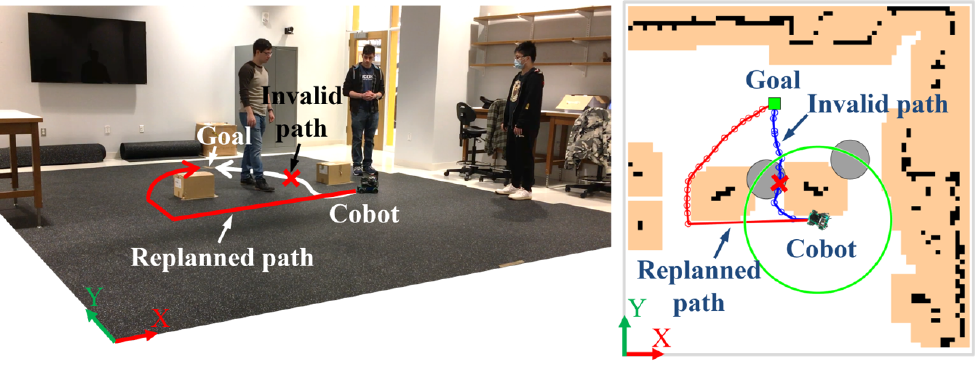}\label{fig:experiment_result_part4}}\hspace{-3pt}\\

    \centering
    \subfloat[Current path is blocked again. Successfully replanned in real time.]{
        \includegraphics[width=0.44\textwidth]{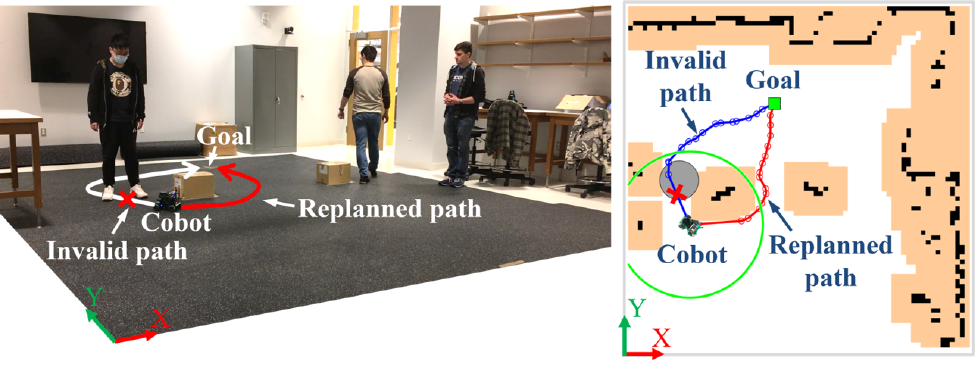}\label{fig:experiment_result_part5}}\hspace{-3pt}\quad
    \centering
    \subfloat[Cobot reaches the goal.]{
         \includegraphics[width=0.44\textwidth]{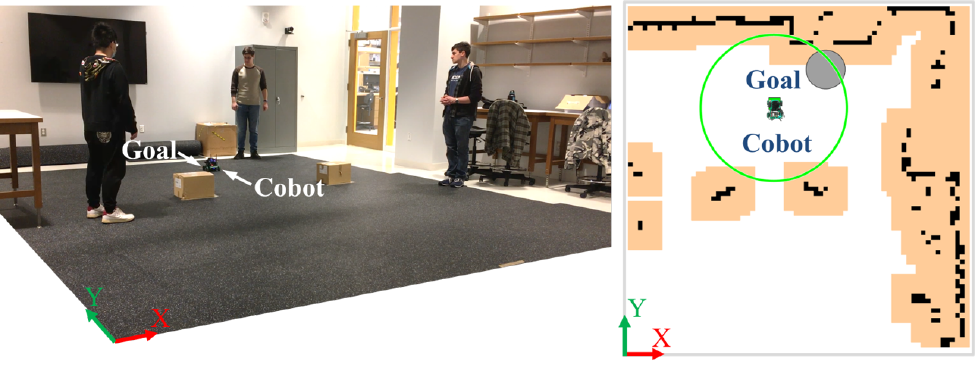}\label{fig:experiment_result_part6}}\hspace{-3pt}\\
         
    \caption{Snapshots of a real experiment in a laboratory with both static and dynamic obstacles. Video available in supplementary documents. }\label{fig:experiment_result}
      \vspace{-6pt}    
 \end{figure*}

Fig.~\ref{fig:factory_result} shows the same trend in Scenario 2. As seen, SMART outperforms all other methods in terms of replanning time, success rate, and the total travel time.

\vspace{-6pt}
\subsection{Validation by Real Experiments}
The SMART algorithm is further validated by real experiments in a $7m \times 7m$ lab space with both static and dynamic obstacles. A cobot called ROSMASTER X3~\cite{robot} is used that is equipped with 1) a RPLIDAR S2L lidar~\cite{robot} with a range of $8m$ for obstacle detection, 2) MD520 motor with encoder~\cite{robot} for detection of rotation angle and linear displacement, and 3) MPU9250 IMU~\cite{robot} for detection of speed, acceleration, and orientation. An Extended Kalman Filter \cite{bar2004estimation} is used to fuse data from the IMU and motor encoder for localization. The space is tiled with $0.1m \times 0.1m$ cells. The occupancy grid mapping algorithm~\cite{thrun2005} is used offline to create a static obstacle map, while the humans are detected in real-time.  
The cobot carries Jetson Nano minicomputer that collects sensor measurements and runs the SMART algorithm for real-time replanning, control and navigation. 
Fig.~\ref{fig:experiment_result} shows the various snapshots of an experiment, where the cobot successfully replans a new path multiple times to avoid obstacles until reaching the goal, thus revealing the effectiveness of SMART. The observed replanning time in experiments is $\sim0.03s$.

\vspace{-6pt}
\section{Conclusions and Future Work}
\label{sec:conclusions}
The paper presents an algorithm, called SMART, for adaptive replanning in dynamic environments. To replan a path, SMART performs risk-based tree-pruning to form multiple disjoint subtrees, then exploits and repairs them at selected hot-spots for speed recovery. It is shown that SMART is computationally efficient and complete. The comparative evaluation with existing algorithms shows that SMART significantly improves the replanning time, success rate, and travel time. Finally, SMART is validated by real experiments. 

With further research SMART has the potential for extending to non-holonomic robots and higher dimensional problems. The challenges include 1) defining configuration space including the rotational space components, 2) identifying efficient strategies for partitioning and sampling in the configuration space, and 3) considering the motion constraints in cost functions and tree reconnections. Further, SMART could be extended to problems with 1) joint time-risk optimization,  2) multi-cobot systems, and 3) coverage path planning~\cite{shen2021ct, GRP09}.

\balance
\bibliographystyle{IEEEtran}
\bibliography{reference}

\end{document}